\pdfoutput=1
\documentclass[11pt]{article}

\usepackage[]{acl}

\usepackage{times}
\usepackage{latexsym}
\usepackage{graphicx}
\usepackage{amsmath}
\usepackage{amsthm}
\newtheorem{theorem}{Theorem}
\usepackage{arydshln}
\usepackage{multirow}

\usepackage[T1]{fontenc}
\usepackage[utf8]{inputenc}

\usepackage{microtype}

\title{Rethinking Negative Sampling for Handling Missing Entity Annotations}

\author{Yangming Li, Lemao Liu, Shuming Shi \\
	Tencent AI Lab \\
	Shenzhen, China \\
  \texttt{\{newmanli,redmondliu,shumingshi\}@tencent.com} \\}

\begin{document}
\maketitle

\begin{abstract}
	
	Negative sampling is highly effective in handling missing annotations for named entity recognition (NER). One of our contributions is an analysis on how it makes sense through introducing two insightful concepts: missampling and uncertainty. Empirical studies show low missampling rate and high uncertainty are both essential for achieving promising performances with negative sampling. Based on the sparsity of named entities, we also theoretically derive a lower bound for the probability of zero missampling rate, which is only relevant to sentence length. The other contribution is an adaptive and weighted sampling distribution that further improves negative sampling via our former analysis. Experiments on synthetic datasets and well-annotated datasets (e.g., CoNLL-2003) show that our proposed approach benefits negative sampling in terms of F1 score and loss convergence. Besides, models with improved negative sampling have achieved new state-of-the-art results on real-world datasets (e.g., EC).
	
\end{abstract}

\section{Introduction}

	With powerful neural networks and abundant well-labeled corpora, named entity recognition (NER) models have achieved promising performances~\citep{huang2015bidirectional,ma-hovy-2016-end,akbik-etal-2018-contextual,li-etal-2020-unified}. However, in many scenarios, available training data is low-quality, which means a portion of named entities are absent in annotations. Fig.~\ref{fig:Example Demo} depicts a sentence and its incomplete annotations. Fine-grained NER~\citep{ling2012fine} is a typical case. Its training data is mainly obtained through applying weak supervision to unlabeled corpora. Past works~\citep{shang-etal-2018-learning,li2021empirical} find missing annotations impact NER models and refer this to \textit{unlabeled entity problem}.
	
	Recently, \citet{li2021empirical} find it's the misguidance of unlabeled entities to NER models in training that causes their poor performances. To eliminate this adverse impact, they propose a simple yet effective approach based on negative sampling. Compared with its counterparts~\citep{li2005learning,tsuboi-etal-2008-training,shang-etal-2018-learning,peng-etal-2019-distantly}, this method is of high flexibility, without relying on external resources, heuristics, etc.
	
	While negative sampling has handled missing annotations well, there is no systematic study on how it works, especially what potential factors are involved. From a number of experiments, we find missampling and uncertainty both worth receiving attention. Missampling means that some unlabeled entities are mistakenly drawn into the set of training negatives by negative sampling. To quantitively describe this, we define missampling rate, the proportion of unlabeled entities in sampled negatives, for a sentence. Uncertainty indicates how hard a sampled negative is for NER models to recognize, and we use entropy to estimate it. Empirical studies show low missampling rate and high uncertainty are both indispensable for effectively applying negative sampling. Besides, based on the observation that entities are commonly sparse, we provide a lower bound for the probability of zero missampling rate with theoretical proof, which is only related to sentence length.

	\begin{figure}
		\centering
		\includegraphics[width=0.48\textwidth]{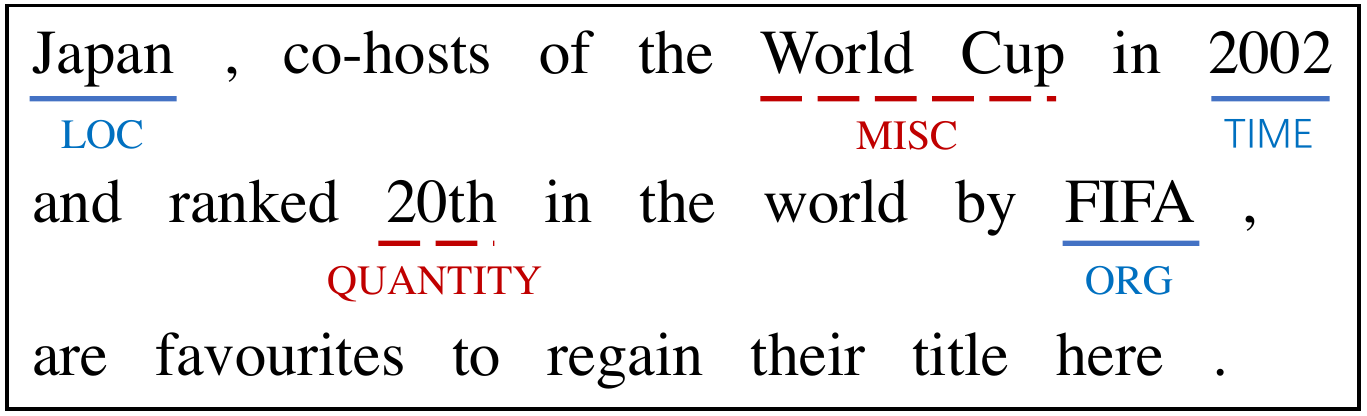}
		
		\caption{A toy example to show \textit{unlabeled entity problem}. The phrases underlined with dashed lines are the named entities neglected by annotators.}
		\label{fig:Example Demo}
	\end{figure}

	\begin{figure*}	
		\centering
		\includegraphics[width=0.9\textwidth]{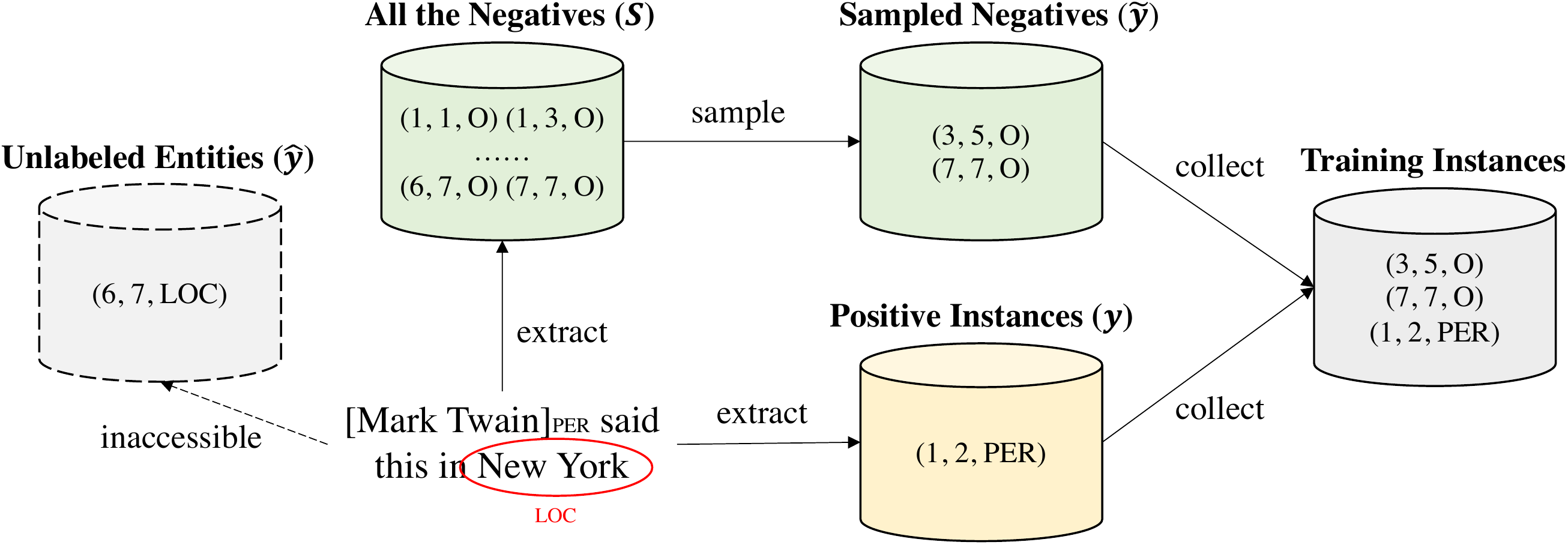}
		
		\caption{An example to depict how negative sampling collects training negatives given an annotated sentence. The phrase marked by a red circle is an unlabeled entity.}
		\label{fig:Negative Sampling}
	\end{figure*}
	
	Originally, \citet{li2021empirical} adopt uniform sampling distribution for negative sampling. Inspired by former findings, we introduce a weighted sampling distribution to displace the uniform one, which takes missampling and uncertainty into account. Our distribution is purely computed from the predictions of an NER model. This means it coevolves with the model throughout the training process. The adaptive property of our method is appealing since it doesn't rely on manual annotations or additional models to indicate valuable negatives.
	
	We have conducted extensive experiments to verify the effectiveness of our weighed sampling distribution. Results on synthetic datasets and well-annotated datasets (e.g., OntoNotes 5.0) show that weighted sampling distribution improves negative sampling in performances and loss convergence. Notably, with improved negative sampling, our NER models have established new state-of-the-art performances on real-world datasets, like EC~\citep{yang-etal-2018-distantly}.
	
\section{Preliminaries}
\label{sec:Preliminaries}

\subsection{Unlabeled Entity Problem}

	Given an $n$-length sentence, $\mathbf{x} = [x_1, x_2, \cdots, x_n]$, an annotator (e.g., human) will mark a set of named entities from it as $\mathbf{y} = \{y_1, y_2, \cdots, y_m\}$. $n$ is sequence length and $m$ is set size. Every entity, $y_k$, of the set, $\mathbf{y}$, is denoted as a tuple, $(i_k, j_k, l_k)$. $(i_k, j_k)$ is the span of the entity that corresponds to the phrase, $\mathbf{x}_{i_k,j_k} = [x_{i_k}, x_{i_k + 1}, \cdots, x_{j_k}]$, and $l_k$ is its label. \textit{Unlabeled entity problem} occurs when some ground truth named entities, $\widehat{\mathbf{y}}$, are missed by annotators, which means they are not contained in the labeled entity collection, $\mathbf{y}$. In distantly supervised NER~\citep{mintz-etal-2009-distant,10.1145/2783258.2783362,fries2017swellshark}, this is resulted from the limited coverage of external resources, such as predefined ontology. In other situations (e.g., fine-grained NER where manual annotation is extremely hard), the cause may be the negligence of human annotators.
	
	Take Fig.~\ref{fig:Negative Sampling} as an example. The set of labeled entities is $\mathbf{y} = [(1, 2, \mathrm{PER})]$, that of unlabeled entities is $\widehat{\mathbf{y}} = \{(6, 7, \mathrm{LOC})\}$, and that of ground-truth entities is $\mathbf{y} \cup \widehat{\mathbf{y}}$.
	
	\begin{figure*}[t]
		\centering
		\includegraphics[width=0.9\textwidth]{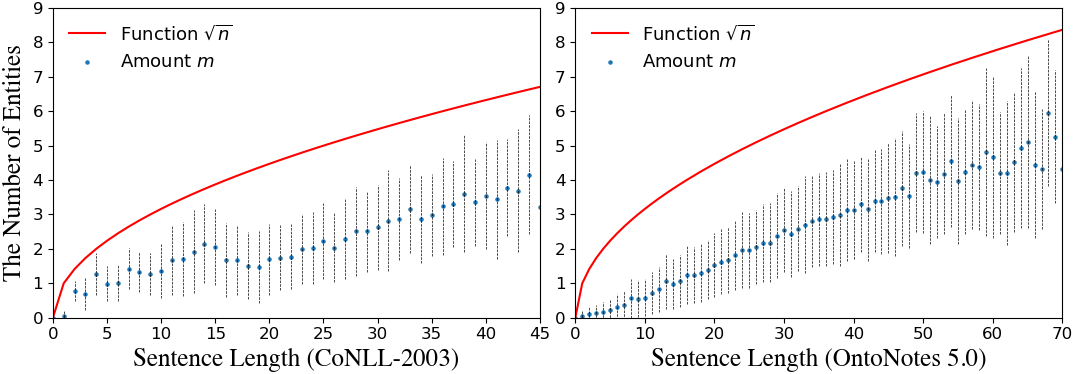}
		
		\caption{The comparisons between changes of entity number and square root curve.}
		\label{fig:Entiy Sparsity}
	\end{figure*}
	
	Let $\mathcal{S}$ denote the set that includes all spans of a sentence, $\mathbf{x}$, except the ones of annotated named entities, $ \mathbf{y}$. Every span in this set is labeled with ``O", indicating that it's a possible negative. A standard training strategy for NER models is to minimize the loss on annotated positives, $\mathbf{y}$, and all negative candidates, $\mathcal{S}$. Unfortunately, since $\mathcal{S}$ might contain unlabeled entities in $\widehat{\mathbf{y}}$, NER models are seriously misguided in training. To address this problem, \cite{li2021empirical} propose to circumvent unlabeled entities with negative sampling.

\subsection{Training with Negative Sampling}
\label{subsec:Negative Sampling}

	The core idea is to uniformly sample a few negative candidates, $\widetilde{\mathbf{y}}$, from $\mathcal{S}$ for reliably training NER models. Under this scheme, the training instances contain sampled negatives, $\widetilde{\mathbf{y}}$, and positives from annotated entities, $\mathbf{y}$. With them, $\mathbf{y}\cup \widetilde{\mathbf{y}}$, a cross-entropy loss is incurred as
	\begin{equation}
	\mathcal{J} = \sum_{(i,j,l) \in \mathbf{y} \cup \widetilde{\mathbf{y}}} -\log P(l\mid \mathbf{x}_{i,j}; \theta).
	\label{eq:loss}
	\end{equation}
	$P(l\mid \mathbf{x}_{i,j}; \theta)$ is the probability that the ground truth label of the span, $(i,j)$, is $l$ and $\theta$ represents the parameters of a model. Following \citet{li2021empirical}, our NER models are all span-based, which treat a span, instead of a single token, as the basic unit for labeling.
	
	Negative sampling is probable to avoid models being exposed to unlabeled entities. As Fig.~\ref{fig:Negative Sampling} shows, the false negative,  $(6, 7, \mathrm{O})$, is not involved in training. \citet{li2021empirical} have empirically confirmed the effectiveness of negative sampling in handling unlabeled entities. However, there is no systematic study to explain how it works, and what factors are relevant.
	
\section{Analyzing Negative Sampling}
\label{sec:two-perspectives}

	We analyze how negative sampling leads NER models that suffer from missing entity annotations to promising results from two angles: missampling and uncertainty.
	
\subsection{Missampling Rate}

\subsubsection{Definition} 
	
	Missampling rate, $\gamma$, is defined as, for a sentence, the proportion of unlabeled entities contained in sampled negatives, $\widetilde{\mathbf{y}}$. Formally, it's computed as
	\begin{equation}\nonumber
	\gamma = 1 - \frac{\#\{ (i,j,l) \mid (i,j,l) \in \widehat{\mathbf{y}}; (i, j, \mathrm{O}) \notin \widetilde{\mathbf{y}}  \}}{\#\widetilde{\mathbf{y}}},
	\end{equation}
	where $\#$ is an operation that measures the size of an unordered set.
	
	The missampling rate $\gamma$ reflects the quality of training instances, $\mathbf{y} \cup  \widetilde{\mathbf{y}}$. A lower averaged rate over the whole dataset means that the NER model meets fewer unlabeled entities in training. Intuitively, this leads to higher F1 scores since there is less misguidance from missing annotations to the model. Hence, missampling is an essential factor for analysis.
	
\subsubsection{Missampling Affects Performance}
	
	\begin{table}[t]
		\centering
		\begin{tabular}{c|ccc}
			\hline
			Average $\gamma$ &0.76\% & 1.52\% & 4.11\% \\ 
			
			\hline
			F1 Score &89.86  &87.35 & 83.11 \\ 
			\hline
		\end{tabular}
	
		\caption{The effects of $\gamma$ on F1.}
		\label{tab:gamma}
	\end{table}

	We design a simulation experiment to empirically verify the above intuition. Like \citet{li2021empirical}, we build synthetic datasets as follows. We start from a well-labeled dataset, i.e., CoNLL-2003~\citep{sang2003introduction}, and then mimic \textit{unlabeled entity problem} by randomly masking manually annotated entities with a fixed probability $p$ (e.g., $0.7$). In this way, we can obtain unlabeled entities, $\widehat{\mathbf{y}}$, and annotated entities, $\mathbf{y}$, for every sentence, $\mathbf{x}$.
	
	We can obtain different pairs of a missampling rate and an F1 score through running a negative sampling based model on different synthetic datasets. Table~\ref{tab:gamma} demonstrates several cases, and we can see the trend that lower missamping rates lead to better performances. Therefore, we conclude that missampling affects the effectiveness of negative sampling.

\subsubsection{Theoretical Guarantee}

	We also theoretically prove that negative sampling is very robust to unlabeled entities based on a natural property of named entities. 
	
	\paragraph{Entity Sparsity.} Unlike other sequence labeling tasks, such as syntactic chunking~\citep{sang2000introduction} and part-of-speech tagging~\citep{schmid-1994-part}, named entities (i.e., non-``O" segments) are commonly sparse in NER datasets. 
	
	Fig.~\ref{fig:Entiy Sparsity} depicts some statistics of two common NER datasets, CoNLL-2003 and OntoNotes 5.0. The blue points are the averaged number of entities for sentences of fixed lengths. Every point stands on the center of a dashed line, whose length is the 1.6 variance of the entity numbers. The red curves are the square roots of sentence lengths. To avoid being influenced by ``rare events" we erase the points supported by too few cases (i.e., 20).
	
	From the above figure, we can see that the number of ground truth named entities (i.e., unlabeled entities, $\widehat{\mathbf{y}}$, and annotated ones, $\mathbf{y}$) in a sentence is generally smaller than the square root of sentence length, $\sqrt{n}$. Empirically, we have $\#\mathbf{y}+ \#\widehat{\mathbf{y}} \le \sqrt{n}$.
	
	\begin{theorem}
		\label{theorem}
		For a $n$-length sentence $\mathbf{x}$ , assume $\widetilde{\mathbf{y}}$ is the set of sampled negatives with size $\lceil{\lambda n}\rceil (0<\lambda<1)$  via negative sampling. If the premise of entity sparsity holds, then the probability of zero missampling rate, i.e., $\gamma=0$, is bounded.
	\end{theorem}
	
	\begin{proof}
	Since $\widetilde{\mathbf{y}}$ is uniformly sampled from $\mathcal{S}$ without replacement, the probability $q$ that $\gamma=0$ for a single sentence $\mathbf{x}$ can be formulated as
	\begin{equation*}  
		q = \prod_{0 \le i < \lceil{\lambda n}\rceil} \Big(1 - \frac{\#\widehat{\mathbf{y}}}{\frac{n(n+1)}{2} - m - i}\Big),
	\end{equation*}
	where $m=\#\mathbf{y}$. The $i$-th product term is the probability that, at the $i$-th sampling turn, the $i$-th sampled candidate doesn't belong to unlabeled entity set, $\widehat{\mathbf{y}}$. % $\frac{n(n-1)}{2} - m$ is the number of all the negative candidates. $i$ is the number of previously sampled candidates. 
	
	Then we can derive the following inequalities:
	\begin{equation*}
		\begin{aligned}
		q & \ge \prod_{0 \le i < \lceil{\lambda n}\rceil} \big(1 - \frac{\sqrt{n} - m}{\frac{n(n+1)}{2} - m - i}\big) \\
  & \ge \prod_{0 \le i < \lceil{\lambda n}\rceil} \big(1 - \frac{\sqrt{n} }{\frac{n(n+1)}{2} - i}\big) \\ 
  & > \big(1 - \frac{2\sqrt{n} }{n(n-1) + 2}\big)^{\lceil{\lambda n}\rceil}
		\end{aligned}.
	\end{equation*}
	The first inequality holds because of the assumption; the second one holds because $\frac{\sqrt{n} - m}{\frac{n(n+1)}{2} - m - i}$ is monotonically decreases as $m$ increases, and $m \ge 0$; the last inequality hold since $\frac{\sqrt{n} }{\frac{n(n+1)}{2} - i}$ increases with decreasing $i$, $i < \lceil{\lambda n}\rceil$, and $\lceil{\lambda n}\rceil \le n$.
	
	Because $(1 + a)^b \ge 1 + ba$ for $a \ge -1 \cap b \ge 1$ and $\lceil{\lambda n}\rceil < \lambda n + 1$, we have
	\begin{equation*}
		\begin{aligned}
		q & > \Big( 1 - \frac{2\sqrt{n}}{n(n-1) + 2} \Big)^{\lceil{\lambda n}\rceil} \\
	 & \ge 1 - \frac{2(\lambda n + 1)\sqrt{n}}{n(n-1) + 2} \\ & > 1 - \frac{4\lambda\sqrt{n}}{n-1}
		\end{aligned}.
	\end{equation*} % 1 - 2\sqrt{n}\frac{(\lambda + \frac{1}{n})}{(n-3) + \frac{2}{n}}
	The right-most term monotonically increases with the sentence length $n$, and thus the probability of zero missampling rate for every sentence has a lower bound.
	\end{proof}
	
	This theorem shows that missampling rates for standard negative sampling are controllable, and implies why negative sampling succeeds in handling missing annotations.
		
\subsection{Uncertainty}

\subsubsection{Definition}
\label{sec:uncertainty}

	Assume $P_o(l \mid \mathbf{x}_{i,j})$ is an oracle model that accurately estimates a label distribution over every span $(i,j)$. The uncertainty is defined as the entropy of this distribution:
	\begin{equation*}
	\begin{aligned}
		& H(L\mid \mathbf{X}=\mathbf{x}_{i,j}) = \\ & \sum_{l \in L}-P_o(l \mid \mathbf{x}_{i,j})\log P_o(l \mid \mathbf{x}_{i,j}),
	\end{aligned}
	\end{equation*}
	where $L$ and $\mathbf{X}$ represent the label space and a span, $\mathbf{x}_{i,j}$, respectively.
	
	Note that the oracle model $P_o(l \mid \mathbf{x}_{i,j})$ is generally unreachable. the common practice is to additionally train a model $P(l \mid \mathbf{x}_{i,j}; \theta)$ (see Sec.~\ref{subsec:Negative Sampling}) to approximate it. Besides, the approximate model is learned on held-out training data to avoid over-confident estimation.
	
	Uncertainties essentially measure how difficult a case is for models to make a decision~\citep{jurado2015measuring}. In active learning, uncertainty is used to mine hard unlabeled instances for human annotator~\citep{settles2009active}.  In our scenario, we suspect that the uncertainty of sampled negatives plays an important role in our training with negative sampling.
	
\subsubsection{Uncertainty Affects Performances}
\label{subsec:uncertainty effect}

	\begin{table}[t]
		\centering
		\begin{tabular}{c|ccc}
			\hline
			$H$ & Top-$k$ & Middle-$k$ & Bottom-$k$ \\ 
			\hline
			F1 Score & 88.82  &87.72 & 85.56  \\
			\hline
		\end{tabular}
		
		\caption{The effects of $H$ on F1.}
		\label{tab:h}
	\end{table}

	We design an empirical experiment to verify our hypothesis. Specifically, we first randomly and equally split the entire training data with masked entities into two parts, and the first part is used to train an oracle model $P_o$. For every sentence $\mathbf{x}$ in the second part, we then sample three subsets from $\mathcal{S}$ as training negatives: the first subset denoted by $\widetilde{\mathbf{y}}^t$ corresponding to the top-$k$ uncertainties, and the second denoted by $\widetilde{\mathbf{y}}^m$ corresponding to  middle-$k$ uncertainties, and the third denoted by $\widetilde{\mathbf{y}}^b$ corresponding to the bottom-$k$ uncertainties, with $k=\lceil{\lambda n}\rceil$. Since missampling affects F1 scores as aforementioned, we eliminate the effect on missampling rate by setting $\gamma=0$ when constructing both subsets, i.e., neither subset contains any spans included in $\widehat{\mathbf{y}}$. Finally, we respectively train three models on top of three negative subsets according to Eq.~\ref{eq:loss}, and report their performances on test data in Table~\ref{tab:h}. We can see that the model trained on $\widetilde{\mathbf{y}}^t$ achieves the best performance, which validates our hypothesis.
	
\section{Improving Negative Sampling}
\label{sec:Further Improvement}

	\begin{figure*}
		\centering
		\includegraphics[width=1.0\textwidth]{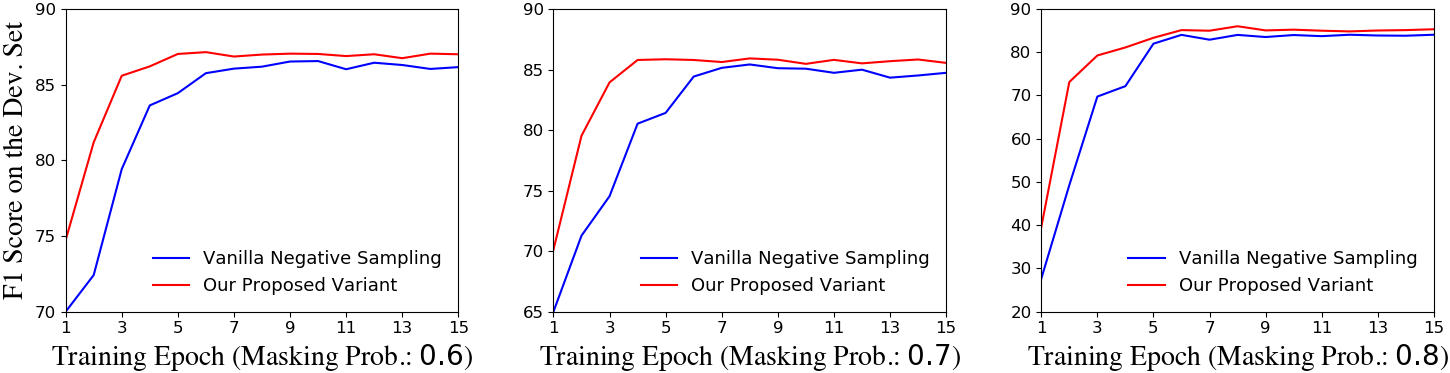}
		
		\caption{The changes of F1 scores with training epochs on some synthetic datasets.}
		\label{fig:Convergence}
	\end{figure*}

	The previous section shows that the effectiveness of negative sampling is dependent on two factors: missampling and uncertainty. As a result, if we had considered both quantities when sampling negatives, we should see larger improvements from final models. In this section, we propose an adaptive and weighted sampling distribution based on these two factors. 
	
	Unfortunately, since missampling rate is defined on top of the unlabeled entities $\widehat{\mathbf{y}}$ which is unknown in practice, it is not straightforward to apply missampling for improving negative sampling.  Therefore, we assume that an oracle model, $z_{i,j,l}=P_o(l\mid \mathbf{x}_{i,j})$, exists, which is likely to predict the ground-truth label for every span $\mathbf{x}_{i,j}$. Then we define a score $v_{i,j}$ as the difference between the score $z_{i,j,\mathrm{O}}$ and the maximum label score on the span $(i,j)$:
	\begin{equation}
	v_{i,j} =  z_{i,j,\mathrm{O}} - \max_{l \in \mathcal{L}} z_{i,j,l}.
	\end{equation}
	Intuitively, if $v_{i,j}$ is high, then $z_{i,j,\mathrm{O}}$ is high and $\max_{l \in \mathcal{L}} z_{i,j,l}$ is low. In other words, $\mathbf{x}_{i,j}$ is likely to be with ``O" label and thus the missampling rate should be small. Hence sampling such a span as a negative won't hurt NER models.  Note that $\max_{l\in\mathcal{L}} z_{i,j,l}$ in the right hand acts as normalization, making $v_{i,j}$ comparable among different spans $(i,j)$.
	
	We also define an uncertainty score, $u_{i,j}$, as the entropy of the label distribution for a span:
	\begin{equation}
	\begin{aligned}
		u_{i,j} & = H(L \mid \mathbf{X} =  \mathbf{x}_{i,j}) \\ & = -\sum_{l \in \mathcal{L}} z_{i,j,l}\log z_{i,j,l}.
	\end{aligned}
	\end{equation}
	As discussed in Sec. \ref{subsec:uncertainty effect}, training a NER model with the negatives of higher uncertainty scores, $u_{i,j}$, brings better performances.
	
	Based on $v_{i,j}$ and $u_{i,j}$, we design the following weighted sampling distribution to displace the uniform one when sampling $k$ negatives from $\mathcal{S}$ without replacement:
	\begin{equation}
	\left\{\begin{aligned}
	r_{i,j} & = u_{i,j} * (1 + v_{i,j})^{\mu} \\
	e_{i,j} & = \frac{\exp(r_{i,j} / T)}{\sum_{(i',j',\mathrm{O}) \in \mathcal{S}} \exp(r_{i',j'} / T)}
	\end{aligned}\right.,
	\end{equation}
	where $T \ge 1$ is a temperature to control the smoothness of sampling distribution. $\mu \ge 1$ is to make a trade-off between $v_{i,j}$ and $u_{i,j}$: a high $\mu$ will ensure a low missampling rate while a low $\mu$ will ensure a high uncertainty score. 
	
	To make our approach practical for use, we should specify how to approximate the oracle model, $P_o(l\mid \mathbf{x}_{i,j})$. In the simulation experiment in Sec.~\ref{sec:uncertainty}, the oracle model is a fixed model via standard negative sampling which is learned on held-out training data. It's natural to use such a fixed model to approximate the oracle model here. However, this will cause a side-effect that our approach is not self-contained due to its dependence on an external model. 
	
	Consequently, we consider an adaptive style: directly using the NER model, $P(l\mid \mathbf{x}_{i,j}; \theta)$, itself as the oracle model whose parameter $\theta$ is learned during the training process. Under this scheme, $T$ is scheduled as $\sqrt{C- c}$, where $C$ is the number of training epochs and $0 \le c < C$ is the current epoch number. Since the NER model $P(l\mid \mathbf{x}_{i,j}; \theta)$ is not accurate in early epochs of training, a more uniform sampling distribution (i.e., higher $T$) is safer for sampling negatives.
	
	\begin{table*}[t]
		\centering
		% \small
		
		\setlength{\tabcolsep}{1.8mm}{}
		\begin{tabular}{c|cc|cc}
			\hline
			
			\multirow{2}{*}{Masking Prob.} & \multicolumn{2}{c|}{CoNLL-2003} & \multicolumn{2}{c}{OntoNotes 5.0} \\
			
			\cline{2-5}
			& Vanilla Neg. Sampling & Our Variant & Vanilla Neg. Sampling & Our Variant  \\
			
			\hline
			0.5 & $89.22$ & $\mathbf{89.51}$ & $88.17$ & $\mathbf{88.31}$ \\
			
			0.6 & $87.65$ & $\mathbf{88.03}$ & $87.53$ & $\mathbf{88.02}$ \\
			
			0.7 & $86.24$ & $\mathbf{86.97}$ & $86.42$ & $\mathbf{86.85}$\\
			
			0.8 & $78.84$ & $\mathbf{82.05}$ & $85.02$ & $\mathbf{86.12}$ \\
			
			0.9 & $51.47$ & $\mathbf{60.57}$ & $74.26$ & $\mathbf{80.55}$ \\
			\hline
			
		\end{tabular}
		\caption{The comparisons of F1 scores on synthetic datasets.} 
		\label{tab:Performances on Synthetic Datasets}
	\end{table*}
	
	\begin{table*}[t]
		\centering
		% \small
		
		\setlength{\tabcolsep}{6.5mm}{}
		\begin{tabular}{c|c|cc}
			
			\hline
			\multicolumn{2}{c|}{Method} & EC & NEWS \\
			
			\hline 
			\multicolumn{2}{c|}{Partial CRF~\citep{yang-etal-2018-distantly}} & $60.08$ & $78.38$ \\
			
			\multicolumn{2}{c|}{Positive-unlabeled (PU) Learning~\citep{peng-etal-2019-distantly}} & $61.22$ & $77.98$ \\
			
			\multicolumn{2}{c|}{Weighted Partial CRF~\citep{jie-etal-2019-better}} & $61.75$ & $78.64$ \\
			
			\hdashline
			\multicolumn{2}{c|}{BERT-MRC~\citep{li-etal-2020-unified}} & $55.72$ & $74.55$ \\
			
			\multicolumn{2}{c|}{BERT-Biaffine Model~\citep{yu-etal-2020-named}} & $55.99$ & $74.57$ \\
			
			\hdashline
			\multirow{2}{*}{\citet{li2021empirical}} & Vanilla Negative Sampling & $66.17$ & $85.39$ \\
			& w/o BERT, w/ BiLSTM & $64.68$ & $82.11$ \\
			
			\hline
			\multirow{2}{*}{This Work} & Our Proposed Variant & $\mathbf{67.03}$ & $\mathbf{86.15}$ \\
			
			& w/o BERT, w/ BiLSTM & $\mathbf{65.81}$ & $\mathbf{83.79}$  \\
			
			\hline
		\end{tabular}
		\caption{The experiment results on two real-world datasets.}
		\label{tab:Performances on Real-world Datasets}
	\end{table*}
	
	Finally, we get a weighted sampling distribution with the NER model, $P(l\mid \mathbf{x}_{i,j}; \theta)$, adaptively approximating the oracle model. Our training procedure is the same as that of vanilla negative sampling (see Fig.~\ref{fig:Negative Sampling}), except for sampling distribution.

\section{Experiments}
\label{sec:Experiments}

	To evaluate our proposed variant (i.e., negative sampling w/ weighted sampling distribution) , we have conducted extensive experiments on under-annotated cases: synthetic datasets and real-world datasets. We also validate its superiority in well-annotated scenarios.
	
\subsection{Settings}

	The well-annotated datasets are CoNLL-2003 and OntoNotes 5.0. CoNLL-2003 contains $22137$ sentences and is split into $14987$, $3466$, and $3684$ sentences for training set, development set, and test set, respectively. OntoNotes 5.0 contains $76714$ sentences from a wide variety of sources. We follow the same format and partition as in~\citet{Luo_Xiao_Zhao_2020}. The construction of synthetic datasets is based on well-annotated datasets and has been already described in Sec.~\ref{sec:two-perspectives}.
	
	Following prior works~\citep{nooralahzadeh-etal-2019-reinforcement,li2021empirical}, we adopt EC and NEWS as the real-world datasets. Both of them are collected by \citet{yang-etal-2018-distantly}. The data contains 2400 sentences annotated by human and is divided into three portions: $1200$ for training set, $400$ for development set, and $800$ for test set. \citet{yang-etal-2018-distantly} build an entity dictionary of size $927$ and apply distant supervision on a raw corpus to get extra $2500$ training cases. NEWS is constructed from MSRA~\citep{levow-2006-third}. Training set is of size $3000$, development set is of size $3328$, and test set is of size $3186$ are all sampled from MSRA. \citet{yang-etal-2018-distantly} collect an entity dictionary of size $71664$ and perform distant supervision on the remaining data to obtain extra $3722$ cases for training. Both EC and NEWS contain massive incomplete annotations. NER models trained on them suffer from \textit{unlabeled entity problem}.
	
	\begin{table*}[t]
		\centering
		% \small
		
		\setlength{\tabcolsep}{4.4mm}{}
		\begin{tabular}{c|cc}
			
			\hline		
			Method & CoNLL-2003 & OntoNotes 5.0 \\
			
			\hline
			Flair Embedding~\citep{akbik-etal-2018-contextual} & $93.09$ & $89.3$ \\
			
			HCR w/ BERT~\citep{Luo_Xiao_Zhao_2020} & $93.37$ & $90.30$ \\
			
			BERT-MRC~\citep{li-etal-2020-unified} & $93.04$ & $91.11$ \\
			
			BERT-Biaffine Model~\citep{yu-etal-2020-named} & $93.5$ & $\mathbf{91.3}$ \\
			
			% LUKE~\citep{yamada-etal-2020-luke} & $\mathbf{94.3}$ & - \\
			
			\hdashline
			Vanilla Negative Sampling~\citep{li2021empirical} & $93.42$ & $90.59$ \\
			
			\hline
			Our Proposed Variant & $\mathbf{93.68}$ & $91.17$ \\
			
			\hline
		\end{tabular}
		\caption{The experiment results on well-annotated datasets.}
		\label{tab:Performances on Well-annotated Datasets}
	\end{table*}

	We adopt the same configurations for all the datasets. The dimensions of scoring layers are $256$. L2 regularization and dropout ratio are $10^{-5}$ and $0.4$, respectively. We set $\mu = 8$. This setting is obtained via grid search. We use Adam~\citep{kingma2014adam} to optimize models. Our models run on GeForceRTX 2080T. At test time, we convert the predictions from our models into IOB format and use conlleval\footnote{https://www.clips.uantwerpen.be/conll2000/chunking/ \\ conlleval.txt.} script to compute the F1 score. In all the experiments, the improvements of our models over the baselines are statistically significant with a rejection probability lower than $0.01$.
	
\subsection{Results on Under-annotated Scenarios}

	We show how NER models with our proposed approach perform on two types of datasets: synthetic datasets (e.g., CoNLL-2003) and real-world datasets (e.g., EC). Synthetic datasets offer us a chance to qualitatively analyze how our approach reacts to changing mask probabilities. For example, we will show that weighted sampling distribution is beneficial in fast loss convergence. Real-world datasets provide more appropriate cases to evaluate NER models, since missing annotations are caused by limited knowledge resources, rather than intentional masking.

\subsubsection{Results on Synthetic Datasets}

	Fig.~\ref{fig:Convergence} shows the changes of F1 scores from vanilla negative sampling and our proposed variant with training epochs. The synthetic datasets are constructed from OntoNotes 5.0. We can see that, compared with vanilla negative sampling, our proposed variant obtains far better performances on the first few epochs and converges much faster. These results clearly verify the superiority of our weighted sampling distribution.
	
	Table~\ref{tab:Performances on Synthetic Datasets} compares vanilla negative sampling with our proposed variant in terms of F1 score. We can draw two conclusions. Firstly, our approach greatly improves the effectiveness of negative sampling. For example, when masking probability $p$ is $0.8$, we increase the F1 scores by $4.07\%$ on CoNLL-2003 and $1.29\%$ on OntoNotes 5.0. Secondly, our variant is still robust when \textit{unlabeled entity problem} is very serious. Setting masking probability $p$ from $0.5$ to $0.9$, our performance on OntoNotes 5.0 only drops by $8.79\%$. By contrast, it's $32.33\%$ for vanilla negative sampling.
	
\subsubsection{Results on Real-world Datasets}

	Real-world datasets contain a high percentage of partial annotations caused by distant supervision. Hence, the models trained on them are faced with serious \textit{unlabeled entity problem}.
	
	Table~\ref{tab:Performances on Real-world Datasets} diagrams the results. The F1 scores of negative sampling and Partial CRF are from their papers. We have additionally reported the results of PU Learning\footnote{https://github.com/v-mipeng/LexiconNER.}, Weighted Partial CRF\footnote{https://github.com/allanj/ner\_incomplete\_annotation.}, BERT-MRC\footnote{https://github.com/ShannonAI/mrc-for-flat-nested-ner.}, and BERT-Biaffine Model\footnote{https://github.com/juntaoy/biaffine-ner.}, using their codes. We can draw three conclusions from the table. Firstly, we can see that BERT-MRC and BERT-Biaffine Model both perform poorly on real-world datasets. This manifests the huge adverse impacts of unlabeled entities on models. Secondly, our variant has achieved new state-of-the-art results on the two datasets. Our scores outnumber those of vanilla negative sampling by $1.30\%$ and $0.89\%$ on them. Thirdly, to make fair comparisons, we also report the results of using Bi-LSTM, instead of BERT, as the sentence encoder. This version still notably surpasses prior methods on the two datasets. For example, compared with Weighted Partial CRF, our improvements are $6.57\%$ on EC and $6.55\%$ on NEWS.
	
\subsection{Results on Well-annotated Scenarios}

	As a by-product, we also evaluate the effectiveness of the proposed method on the well-annotated datasets CoNLL-2003 and OntoNotes 5.0. As shown in Table~\ref{tab:Performances on Well-annotated Datasets}, we have achieved excellent performances on well-annotated datasets. The F1 scores of baselines are copied from \citet{li2021empirical}. With our weighted sampling distribution, the results of negative sampling are improved by $0.28\%$ on CoNLL-2003 and $0.64\%$ on OntoNotes 5.0. Our model even outperforms BERT-Biaffine Model by $0.19\%$ on CoNLL-2003. Compared with a strong baseline, Flair Embedding, our improvements of F1 scores are $0.63\%$ and $2.09\%$ on the two datasets. These results further verify the effectiveness of the proposed sampling distribution.
	
	The comparison here is in fact unfair for our model, because negative sampling only utilizes a small part of negatives, $\lceil{\lambda n}\rceil$ rather than $\frac{n(n+1)}{2} - m$ (see Sec.~\ref{sec:Preliminaries} for the details of these numbers). We also have tried using all the negatives for training our model, and found the resulting performances significantly outnumber those of baselines. The purpose of Table~\ref{tab:Performances on Well-annotated Datasets} is to confirm that negative sampling even works well for situations with complete entity annotations.
	
\section{Related Work}
	
	A number of NER models~\citep{lample-etal-2016-neural,akbik-etal-2018-contextual,clark-etal-2018-semi,li-etal-2020-handling,yu-etal-2020-named} based on end-to-end neural networks and well-labeled data have achieved promising performances. A representative work is Bi-LSTM CRF~\citep{huang2015bidirectional}. However, in many situations (e.g., distantly supervised NER), these seemingly perfect models severely suffer from \textit{unlabeled entity problem}, where massive named entities are not annotated in training data. There are some techniques developed by earlier works to mitigate this issue. Fuzzy CRF and AutoNER~\citep{shang-etal-2018-learning} allow NER models to learn from high-quality phrases that might be potential named entities. Mining these phrases demands external resources~\citep{shang2018automated}, which is not flexible for practical usage. Moreover, there is no guarantee that unlabeled entities are fully covered by these phrases. PU Learning~\citep{peng-etal-2019-distantly,mayhew-etal-2019-named} adopts a weighted training loss and assigns low weights to false negative instances. This approach is limited by requiring prior information or heuristics. Partial CRF~\citep{yang-etal-2018-distantly,jie-etal-2019-better} is an extension of CRF, which marginalizes the loss over all candidates that are compatible with the incomplete annotation. While being theoretically attractive, this approach still needs a portion of well-annotated data to obtain true negatives, which limits its use in real-world applications. For example, in fine-grained NER~\citep{ling2012fine}, all the training data are produced through weak supervision, and its manual annotation is very difficult, so obtaining enough high-quality data is not practical.
	
	Recently, \citet{li2021empirical} find that unlabeled entities severely misguide the NER models during training. Based on this observation, they introduce a simple yet effective approach using negative sampling. It's much more flexible than other methods, without resorting to external resources, heuristics, etc. However, \citet{li2021empirical} haven't well explained why negative sampling works and there are weaknesses in their principle analysis. In this paper, we first show two factors that affect how negative sampling avoids NER models from being impacted by missing annotations. Notably, a theoretical guarantee is provided for the zero missampling rate. Then, we propose weighted sampling distribution to further improve negative sampling based on our former findings.
	
\section{Conclusion} 
	
	In this work, we have made two contributions. On the one hand, we analyze why negative sampling succeeds in handling \textit{unlabeled entity problem} from two perspectives: missampling and uncertainty. Empirical studies show both  low missampling rates and high uncertainties are essential for applying negative sampling. Based on entity sparsity, we also provide a theoretical lower bound for the probability of zero missampling rate. On the other hand, we propose an adaptive and weighted sampling distribution that takes missampling and uncertainty into account. We have conducted extensive experiments to verify whether this further improves the effectiveness of negative sampling. Results on synthetic datasets and well-annotated datasets show that our approach benefits in performances and loss convergence. With improved negative sampling, our NER models also have achieved new state-of-the-art results on real-world datasets.

\bibliography{anthology,custom}
\bibliographystyle{acl_natbib}

\end{document}